\documentclass[conference]{IEEEtran}
\IEEEoverridecommandlockouts
\usepackage{cite}
\usepackage{amsmath,amssymb,amsfonts}
\usepackage{graphicx}
\usepackage{textcomp}
\usepackage{xcolor}
\def\BibTeX{{\rm B\kern-.05em{\sc i\kern-.025em b}\kern-.08em
    T\kern-.1667em\lower.7ex\hbox{E}\kern-.125emX}}

\usepackage{hyperref}  
\usepackage{footnote}
\usepackage{amsthm}
\usepackage{bm}
\usepackage{graphicx}
\usepackage{subcaption}
\graphicspath{{figs/}}
\usepackage{cleveref}
\usepackage{hhline}
\usepackage{wrapfig}
\newtheorem{theorem}{Theorem}

\newcommand{\m}{\mathbf{m}}

\newcommand{\X}{\mathbf{X}}
\newcommand{\x}{\mathbf{x}}

\newcommand{\y}{\mathbf{y}}

\newcommand{\z}{\mathbf{z}}
\usepackage{algorithm}
\usepackage{algpseudocode}

\begin{document}

\title{Multiple Imputation via Generative Adversarial Network for High-dimensional Blockwise Missing Value Problems
\thanks{This work is partly supported by NIH grant R01GM124111.}
}

\author{\IEEEauthorblockN{ Zongyu Dai}
\IEEEauthorblockA{\textit{Department of AMCS} \\
\textit{University of Pennsylvania}\\
Philadelphia, USA \\
daizy@sas.upenn.edu}
\and
\IEEEauthorblockN{Zhiqi Bu}
\IEEEauthorblockA{\textit{Department of AMCS} \\
\textit{University of Pennsylvania}\\
Philadelphia, USA \\
zbu@sas.upenn.edu}
\and
\IEEEauthorblockN{ Qi Long}
\IEEEauthorblockA{\textit{Division of Biostatistics} \\
\textit{University of Pennsylvania}\\
Philadelphia, USA \\
qlong@upenn.edu}
}

\maketitle

\begin{abstract}
Missing data are present in most real world problems and need careful handling to preserve the prediction accuracy and statistical consistency in the downstream analysis. As the gold standard of handling missing data, multiple imputation (MI) methods are proposed to account for the imputation uncertainty and provide proper statistical inference. 

In this work, we propose Multiple Imputation via Generative Adversarial Network (MI-GAN), a deep learning-based (in specific, a GAN-based) multiple imputation method, that can work under missing at random (MAR) mechanism with theoretical support. MI-GAN leverages recent progress in conditional generative adversarial neural works and shows strong performance matching existing state-of-the-art imputation methods on high-dimensional datasets, in terms of imputation error. In particular, MI-GAN significantly outperforms other imputation methods in the sense of statistical inference and computational speed.
\end{abstract}


\section{Introduction}
Missing values are common in almost all real datasets and they have a far-reaching impact on the data analysis. For example, integrated data from multiple sources are often analyzed in areas such as the financial analysis and the biomedical research. Since each source only collects a subset of features for its samples, and different sources may have different subsets of features, the blockwise missing data often arise and pose challenges in the downstream analysis. As a concrete example, consider 4 hospitals that collect the test results related to a certain disease (see \Cref{table:4 pattern}). While the first hospital can run all the tests for its patients, the second hospital can only run the first 4 tests; the third hospital can only run the first 3 tests and the fourth hospital is capable of running all but the 4th test. After integrating all the patient data across the hospitals, the final dataset contains blockwise missing values.

To deal with this blockwise missing data pattern, it is not sufficient to do the complete case analysis (which discards all samples with missing values and often leads to improper inference and biased findings in the subsequent analysis). Instead, we apply imputation methods to fill in the missing values and conduct inference on the imputed datasets for better accuracy and proper statistical inference. Generally speaking, different imputation methods are proposed to work under different missing mechanisms from which the missing data are generated, which include missing completely at random (MCAR), missing at random (MAR), and missing not at random (MNAR). To be specific, MCAR means that the missing probabilities for each entry (or sample) are the same, independent of the values; MAR means that the missing probabilities depend on the observed values, but not on the missing values; MNAR means that the missing probabilities can depend on both the observed and the missing values. In practice, MCAR is the easiest setting where all methods are supposed to work. MNAR is the most difficult setting where no imputation methods work provably without additional structure assumptions. In this work, we develop MI-GAN, a valid imputation method that works on the MAR (and MCAR) mechanism with the theoretical support.

There are two classes of imputation methods, depending on whether the missing data are imputed for one or multiple times, which are referred to as the single imputation (SI) and the multiple imputation (MI), respectively \cite{little2019statistical}. Single imputation methods, such as the matrix completion \cite{mazumder2010spectral,srebro2004maximum,hastie2015matrix,chen2019inference}, often underestimate the uncertainty of the imputed values and cause bias in the downstream analysis. In comparison, multiple imputation methods, such as MICE \cite{van2007multiple,buuren2010mice,deng2016multiple,zhao2016multiple} and our MI-GAN, can overcome this shortage by adequately accounting for the uncertainty of imputed values through the Rubin's rule \cite{little2019statistical}.

In terms of the learning models, state-of-the-art imputation methods are generally categorized into chained equation-based methods \cite{van2007multiple,buuren2010mice,deng2016multiple,zhao2016multiple}, random forest-based methods \cite{stekhoven2012missforest}, joint modeling \cite{schafer1997analysis,garcia2010pattern,liang2018imputation,zhao2020missing}, matrix completion \cite{mazumder2010spectral,srebro2004maximum,hastie2015matrix,chen2019inference}, and deep learning-based methods \cite{vincent2008extracting,gondara2018mida,ivanov2018variational,mattei2019miwae,yoon2018gain,zhang2018medical,li2019misgan,lee2019collagan}. Chained equation-based methods including MICE are arguably the most popular imputation methods due to their practically stable imputation performance in the low dimensional setting. Additionally, MICE is generally regarded as an MAR method despite its lack of theoretical guarantees. However, chained equation-based methods can be extremely time-consuming in the high-dimensional setting and their performance often deteriorates significantly as the feature dimension increases. Similarly, MissForest \cite{stekhoven2012missforest} is a popular random forest-based imputation method that suffers from the same issue as MICE. Joint modeling-based methods often assume data are generated from Gaussian distribution and they usually have solid theoretical guarantees under MAR mechanism. Nevertheless, joint modeling-based methods' performance also deteriorates significantly in high dimension or when their assumptions are violated in practice. Matrix completion methods, such as SoftImpute \cite{mazumder2010spectral}, conduct the single imputation based on the low-rank assumption and hence usually lead to improper inference. Recently, many deep learning-based imputation methods have been proposed, such as GAIN \cite{yoon2018gain} and optimal transport-based methods \cite{muzellec2020missing}. Specifically, GAIN is a novel multiple imputation method that does not assume the existence of the complete cases. On one hand, GAIN may empirically work for some datasets under MCAR and MAR mechanisms. On the other hand, GAIN is only theoretically supported under MCAR mechanism, unlike our MI-GAN which is supported under MAR. Optimal transport-based methods including the Sinkhorn and Linear RR (both from \cite{muzellec2020missing}) have shown empirical outerperformance over other imputation methods under MCAR, MAR and MNAR mechanism, yet the strong performance no longer holds true in the high dimension. 

\textbf{Our contribution:} In this paper, we propose two novel GAN-based multiple imputation methods, namely $\text{MI-GAN}_1$ and $\text{MI-GAN}_2$, which can work for high dimensional blockwise pattern of missing data with a moderate sample size. We highlight that $\text{MI-GAN}_1$ is equipped with theoretical guarantees under the MAR mechanism. Importantly, we further propose $\text{MI-GAN}_2$ to boost the empirical performance through an iterative training that leverages all the cases, in constrat to $\text{MI-GAN}_1$ which only utilizes the complete cases. Extensive synthetic and real data experiments demonstrate that MI-GANs outerperform other state-of-the-art imputation methods in statistical inference, computational speed, and scalability to high dimension, and perform comparably in terms of imputation error.













\section{MI-GAN}

\begin{figure}

  \centering
  \includegraphics[width=\linewidth,height=4cm]{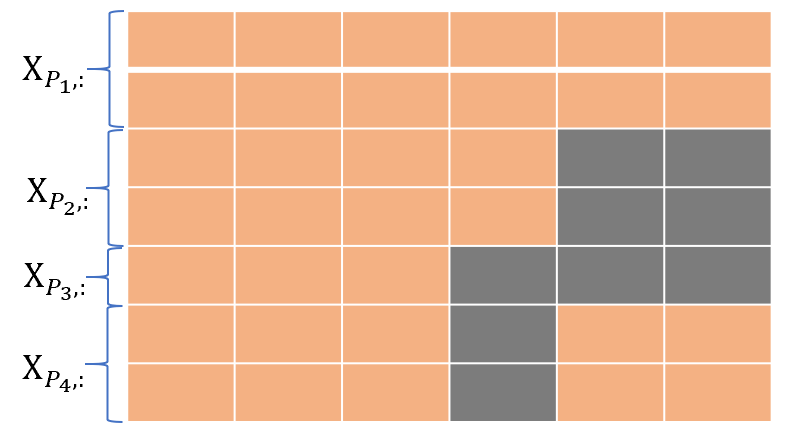}
  \caption{Multivariate 4-pattern missing data. Orange squares represent observed data and gray squares represent missing data.}
 \label{table:4 pattern}
\vspace{-0.3cm}
\end{figure}

We start with a description of multivariate-$K$ pattern missing data. 
We assume there are $n$ samples/cases, each containing $p$ features/variables which are possibly missing. Throughout this paper, we consider high dimensional settings, i.e., $p>n$. Let matrix $\X \in \mathbb{R}^{n\times p}$ denote the data matrix and $\X_{i,j}$ represents the value of the $j$-th variable for the $i$-th sample. Additionally, $\X_{i,:}$ and $\X_{:,j}$ stand for the $i$-th row vector and the $j$-th column vector, respectively. It is true that any missing data can be pre-processed and grouped into $K$ patterns $\X_{P_k,:}$ for $k\in[K]$, where the samples within each pattern have the same observed and missing features denoted by the index set $\textbf{obs}(k)$ and the index set $\textbf{mis}(k)$, respectively. Here $P_k$ is the \textit{index set} for the rows in $\X$ which belong to the $k$-th pattern. Without loss of generality, we let $\X_{P_1,:}$ denote the set of complete cases for which all features are observed. Furthermore, we define $\X_{P_{-k},:}=\X\backslash\X_{P_k,:}$ as the complement data matrix for $\X_{P_k,:}$. See the example in \Cref{table:4 pattern}, where the incomplete matrix contains $7$ samples which can be grouped into $4$ patterns. The samples in the first pattern are all complete and remaining samples contain missing values. Here the pattern index sets are $P_{1}=\{1,2\}$, $P_{2}=\{3,4\}$, $P_{3}=\{5\}$, $P_{4}=\{6,7\}$, and observed feature index sets are $\textbf{obs}(1)=\{1,2,3,4,5,6\}$, $\textbf{obs}(2)=\{1,2,3,4\}$, $\textbf{obs}(3)=\{1,2,3\}$, $\textbf{obs}(4)=\{1,2,3,5,6\}$.  

Without causing confusion, we let $\x_{\textbf{obs}(k)}$ and $\x_{\textbf{mis}(k)}$ denote the observed variables and the missing variables of the $k$-th pattern. Additionally, we define $K$ \textit{mask vectors} $\m_k \in \mathbb{R}^p$ for each pattern $k \in [K]$: $\m_k(j)=1$ if $j\in \textbf{obs}(k)$ otherwise $\m_k(j)=0$.

\subsection{$\text{MI-GAN}_1$:Direct Imputation} \label{MIGAN1}

Here our goal is to impute the missing values in each pattern. In particular, we aim to generate imputed values from $f(\x_{\textbf{mis}(k)}|\x_{\textbf{obs}(k)})$, the conditional distribution of missing variables given observed variables in the $k$-th pattern. At the high level, $\text{MI-GAN}_1$ is an ensemble of $(K-1)$ GANs which are composed of $(K-1)$ pairs of generators and discriminators, and are trained only on complete cases (which belong to the first pattern). Each GAN is used to model one conditional distribution $f(\x_{\textbf{mis}(k)}|\x_{\textbf{obs}(k)})$. \Cref{migan1} shows the architecture of our $\text{MI-GAN}_1$. The details of $\text{MI-GAN}_1$ are described as follows.

\begin{figure*}[hbt!]
\centering
\includegraphics[width=14cm]{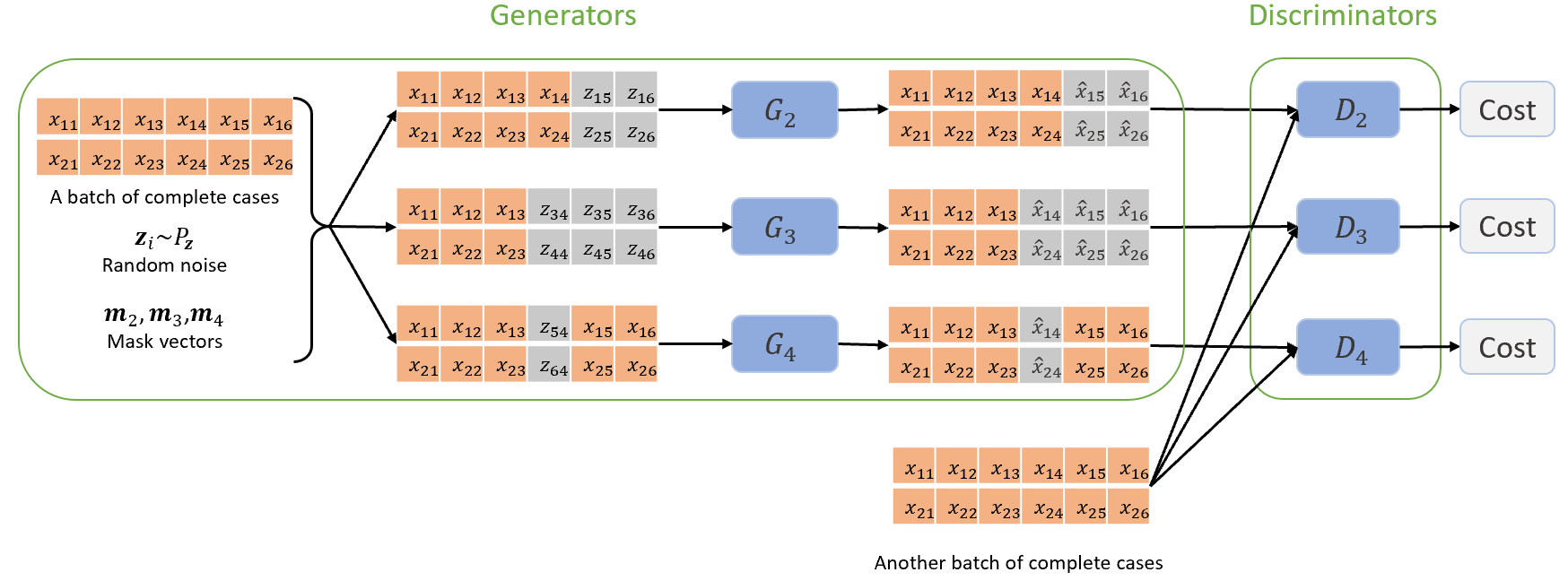}
\caption{$\text{MI-GAN}_1$ applied to the four-pattern missing data in \Cref{table:4 pattern}}
\label{migan1}
\end{figure*}

\textbf{Generator $G_k$:} The $k$-th generator $G_k$ is designed to impute the missing values in the $k$-th pattern. Let $\x$ denotes a complete case in $\X_{P_1,:}$ and $\z$ denotes an independent $p$-dimensional noise from $\mathcal{N}(0,\mathbf{I})$. Then $G_k: \mathbb{R}^p\times \mathbb{R}^p\times\{0,1\}^p\to \mathbb{R}^p$ is a function which takes complete case $\x$, noise $\z$ and mask $\m_k$ as input and outputs a vector of imputations. Here the basic idea is to replace the values of $\x$ in the covariate set $\textbf{mis}(k)$ with a random noise, then feed this partially-true noisy data into the generator to obtain a high-quality imputation. Specifically, for $\text{MI-GAN}_1$, the generator $G_k$ entails two steps: 
\begin{itemize}
    \item $\widehat{G}_k(\x,\z,\m_k) = \widehat{G}_k(\x\odot\m_k +\z\odot(1-\m_k))$ is vector of length $p$, where $\odot$ denotes the element-wise multiplication and  $\widehat{G}_k$ is the generator network that outputs a value for every covariate, even its value was observed.

\item Replace $\widehat{G}_k(\x\odot\m_k +\z\odot(1-\m_k))$ at covariate set $\textbf{obs}(k)$ with true values to ensure that the observed values are intact and we only impute the missing values. Hence the output of generator is $G_k(\x,\z,\m_k)=\widehat{\x}=\x\odot\m_k +\widehat{G}_k(\x,\z,\m_k)\odot(1-\m_k)$
\end{itemize}

\textbf{Discriminator $D_k$:} 
Denote the output distribution of $k$-th generator $G_k$ as $P_k$ and the distribution of complete cases as $P$. Then discriminator $D_k$ takes $\widehat{\x}\sim P_k$ for $k \in [K]$ and $\x \sim P$ as input. The design of the discriminator depends on the training algorithm used. Here we use Wasserstein GAN \cite{arjovsky2017wasserstein,gulrajani2017improved} framework to train $\text{MI-GAN}_1$. Hence the discriminator $D_k: \mathbb{R}^{p} \to \mathbb{R}$ is a $1$-Lipschitz function which estimates the Wasserstein-1 distance between $P$ and $P_k$ using $\widehat{\x}$ and $\x$.

\textbf{Objective Function:} We train the discriminator $D_k$ to estimate the Wasserstein-1 distance (Earth-Mover distance) between $P_k$ and $P$, and we simultaneously train the generator $G_k$ to minimize this distance. For the $k$-th pattern, we consider the following objective function,
\begin{align*}
     L(D_k,G_k) &= \mathbb{E}_{\widehat{\x}\sim P_k}[D_k(\widehat{\x})]-\mathbb{E}_{\x\sim P}[D_k(\x)].
\end{align*}


Then, we ensemble the losses of every GAN with equal weights,
\begin{align*}
    L(D_2,G_2,\cdots,D_K, G_K) = \sum_{k=2}^K L(D_k,G_k).
\end{align*}

Hence the objective of $\text{MI-GAN}_1$ is the minimax problem given by 
\begin{align}
\label{eq: loss}
\min_{G_2,\cdots,G_K} \max_{D_2,\cdots,D_K} L(D_2,G_2,\cdots,D_K, G_K).
\end{align}


\subsection{Theoretical Properties of $\text{MI-GAN}_1$} \label{theorectical properties}

We provide a theoretical analysis of \Cref{eq: loss} by considering a simplified setting. Let $R_k$ denotes the dummy variable for the $k$-th pattern ($k\in [K]\backslash \{1\}$). Hence $R_k=1$ with $R_j=0$ (for $\forall j\in [K]\backslash \{1\}$ and $j\neq k$) means that this case belongs to the $k$-th pattern ($k\in [K]\backslash \{1\}$), and $R_k=0$ (for $\forall k\in [K]\backslash \{1\}$) means this case is complete. 


We denote $\textbf{obs}=\cap_{k=1}^K \{\textbf{obs}(k)\}$ and $\textbf{mis}=\cup_{k=1}^K \{\textbf{mis}(k)\}$ as the set of commonly observed variables and the set of possibly missing variables across all patterns, respectively, and suppose $\textbf{obs}\neq \varnothing$. For the example in \Cref{table:4 pattern},  $\textbf{obs}=\{1,2,3\}$ and $\textbf{mis}=\{4,5,6\}$. Throughout this section, we work with the MAR mechanism such that $\x_{\textbf{mis}} \perp \!\!\! \perp R_k |\x_{\textbf{obs}}$ for any $k\in [K]\backslash \{1\}$, which means, given $\x_{\textbf{obs}}$, $\x_{\textbf{mis}}$ is independent with $R_k$.

\begin{theorem} \label{themrem 1}
Suppose the generators $G_2,\cdots,G_K$ and discriminators $D_2,\cdots,D_K$ are the optimal solutions of \Cref{eq: loss}, then each generator $G_k$ learns the true conditional distribution $f(\x_{\textbf{mis}(k)}|\x_{\textbf{obs}(k)})$ under our MAR setting.
\end{theorem}
\begin{proof}[Proof of \Cref{themrem 1}]
Let $\theta_k$ denote the optimal parameters of the generator $G_k$ and $p_{\theta_k}$ denote the probability density of the imputation distribution $\hat{\x}_{\textbf{mis}(k)}$ conditioned on $\x_{\textbf{obs}(k)}$. When the optimal discriminator is achieved, it perfectly estimates the Wasserstein-1 distance between the distribution of complete cases and the distribution of $G_k$'s outputs (c.f. \cite[Theorem 3]{arjovsky2017wasserstein}). Hence this Wasserstein-1 distance is zero, which means complete case distribution equals to $G_k$'s output distribution $P=P_k$. Denote $\phi$ as some probability density derived from the complete case distribution. Then for all $(t_1,t_2)$, the density function of $G_k$'s output distribution $P_k$ is 
\begin{align*}
    &\phi(\x_{\textbf{obs}(k)}=t_1,\hat{\x}_{\textbf{mis}(k)}=t_2|\cap_{k=2}^K \{R_k=0\})
    \\
    =&\phi(\x_{\textbf{obs}(k)}=t_1|\cap_{k=2}^K \{R_k=0\})p_{\theta_k}(\hat{\x}_{\textbf{mis}(k)}=t_2|\x_{\textbf{obs}(k)}=t_1)
\end{align*}
where the second part of last equation comes from $G_k$'s output depends on its input, and $\cap_{k=2}^K \{R_k=0\}$ represents $R_2=0,\cdots,R_K=0$ are all satisfied. For all $(t_1,t_2)$, the density function of complete case distribution $P$ is
\begin{align*}
    &\phi(\x_{\textbf{obs}(k)}=t_1,\x_{\textbf{mis}(k)}=t_2|\cap_{k=2}^K \{R_k=0\})
    \\
    =&\phi(\x_{\textbf{obs}(k)}=t_1|\cap_{k=2}^K \{R_k=0\})\phi(\x_{\textbf{mis}(k)}=t_2|\x_{\textbf{obs}(k)}=t_1)    
\end{align*}
where the last equation comes from the definition of MAR mechanism: $\x_{\textbf{mis}} \perp \!\!\! \perp R_k |\x_{\textbf{obs}}$.
Therefore, we can conclude
 \begin{align*}
   p_{\theta_k}(\widehat{\x}_{\textbf{mis}(k)}=t_2|\x_{\textbf{obs}(k)}=t_1)=  \phi(\x_{\textbf{mis}(k)}=t_2|\x_{\textbf{obs}(k)}=t_1)
 \end{align*}
which means the optimal $G_k$ learns the true conditional distribution.
\end{proof}

Now we investigate the distribution of imputed samples. by using the optimally trained generator $G_k$ to impute the missing values in the $k$-th pattern. 
\begin{theorem}\label{theorem 2}
Suppose the optimal generators $G_2,\cdots,G_K$ and the discriminators $D_2,\cdots,D_K$ are the optimal solutions of \Cref{eq: loss}, then the imputed incomplete cases in the $k$-th pattern follow the true distribution $f(\x_{\textbf{obs}(k)},\x_{\textbf{mis}(k)}|R_k=1)$. 
\end{theorem}
\begin{proof}[Proof of \Cref{theorem 2}]
From \Cref{themrem 1}, we know $$p_{\theta_k}(\widehat{\x}_{\textbf{mis}(k)}=t_2|\x_{\textbf{obs}(k)}=t_1)=  \phi(\x_{\textbf{mis}(k)}=t_2|\x_{\textbf{obs}(k)}=t_1).$$
Hence for all $(t_1,t_2)$,
\begin{align*}
   &\phi(\x_{\textbf{obs}(k)}=t_1,\hat{\x}_{\textbf{mis}(k)}=t_2|R_k=1)\\=&\phi(\x_{\textbf{obs}(k)}=t_1|R_k=1)p_{\theta_k}(\hat{\x}_{\textbf{mis}(k)}=t_2|\x_{\textbf{obs}(k)}=t_1)\\=&\phi(\x_{\textbf{obs}(k)}=t_1|R_k=1)\phi(\x_{\textbf{mis}(k)}=t_2|\x_{\textbf{obs}(k)}=t_1)\\=&\phi(\x_{\textbf{obs}(k)}=t_1,\x_{\textbf{mis}(k)}=t_2|R_k=1)
\end{align*}
where the first term denotes the density of imputed samples' distribution in the $k$-th pattern and the last term denotes the density of the $k$-th pattern sample distribution.   

\end{proof}

\subsection{$\text{MI-GAN}_1$ Algorithm}

In this section, we provide details of $\text{MI-GAN}_1$ algorithm. At the high level, we use an approach similar to that in WGAN with gradient penalty \cite{gulrajani2017improved}, and solve the minimax optimization problem (\Cref{eq: loss}) in an iterative manner.

When training the GAN for the $k$-th pattern missing data, we first optimize the discriminator $D_k$ with a fixed generator $G_k$. Notably, we apply the gradient penalty to the loss function to enforce the Lipschitz constraint \cite{gulrajani2017improved}:
\begin{align*}
     L(D_k) &=\mathbb{E}_{\widehat{\x}\sim P_k}[D_k(\widehat{\x})]-\mathbb{E}_{\x\sim P}[D_k(\x)] \\&+\lambda_1\cdot \mathbb{E}_{\widetilde{\x}}[(\|\nabla_{\widetilde{\x}} D(\widetilde{\x})\|_2-1)^2]
\end{align*}
where $\lambda_1$ is a hyperparameter, $\widetilde{\x}= \epsilon \cdot \x +(1-\epsilon)\cdot \widehat{\x}$ with $\x\sim P$, $\widehat{\x}\sim P_k$ and $\epsilon \sim U[0,1]$. 

Second, we optimize the generator $G_k$ with newly updated discriminator $D_k$. Notice that the output of the generator network $\widehat{G}_k(\x,\z,\m_k)$ is a vector of the same length with $\x$. Moreover, $\text{MI-GAN}_1$ is trained on complete cases such that $\x$ is fully observed. Thus we add a reconstruction error term to the loss function of $G_k$ to encourage $\widehat{G}_k(\x,\z,\m_k)$ to be close to $\x$. Specifically, 
\begin{align*}
    L(G_k) &=\mathbb{E}_{\widehat{\x}\sim P_k}[D_k(\widehat{\x})] \\
    &+\lambda_2 \cdot \mathbb{E}_{\x \sim P,\z\sim \mathcal{N}(0,\mathbf{I})} [\|\x-\widehat{G}_k(\x,\z,\m_k)\|_1] 
\end{align*}
where $\lambda_2$ is a hyperparameter and each element of $\z$ is independently drawn from Gaussian noise $\mathcal{N}(0,1)$. 

After the training process converges, we arrive at the imputation phase. We feed $\X_{P_k,:}$, $\m_k$ and random noise into the well-trained generator $G_k$ to impute the $k$-pattern missing data. Details are presented in \Cref{MI-GAN1}. When conducting multiple imputation, we run \Cref{MI-GAN1} for multiple times to account for the uncertainty of parameters in the imputation models (weights and biases of $G_k$), and combine the imputations with Robin's rule.

\begin{algorithm}[]
\caption{$\text{MI-GAN}_1$: direct imputation}\label{MI-GAN1}
  \textbf{Input:} $K$-pattern missing data $\X$, gradient penalty coefficient $\lambda_1$, reconstruction error penalty $\lambda_2$, initial parameters for the $K-1$ generators $\theta_{G_2},\cdots,\theta_{G_K}$, initial parameters for the $K-1$ discriminators (critics) $\theta_{D_2},\cdots,\theta_{D_K}$, batch size $m$, number of iterations of the critic per generator iteration $n_{\text{critic}}$, Adam hyperparameters $\alpha,\beta_1,\beta_2$ \\
 \textbf{Output:} Imputed matrix
\begin{algorithmic}[1]

\While {training loss has not converged}

\State{\texttt{\#\# Discriminator Optimization}}
\For{\texttt{$j \in \{1,\dots,n_{\text{critic}}\}$}}
\For{\texttt{$i \in \{1,\dots, m\}$}}
\State Sample two complete cases $\x,\x'$ from $\X_{P_1,:}$, sample $(K-1)$ noise vectors $\z_k \sim \mathcal{N}(0,\mathbf{I})$ for $k\in[K]\backslash\{1\}$, and a random number $\epsilon\sim U[0,1]$

\For{\texttt{$k \in \{2,\dots,K\}$}}

\State $\widehat{\x}_k\leftarrow G_k(\x,\z,\m_k)$
\State $\widetilde{\x}_k\leftarrow \epsilon \x' +(1-\epsilon) \widehat{\x}_k$

\State $L^{(i)}_k\leftarrow D_k(\widehat{\x}_k)-D_k(\x')+\lambda_1(\|\nabla_{\widetilde{\x}_k} D(\widetilde{\x}_k)\|_2-1)^2$
\EndFor
\EndFor

\For{\texttt{$k \in \{2,\dots,K\}$}}
\State $\theta_{D_k} \leftarrow \text{Adam}(\nabla_{\theta_{D_k}}\frac{1}{m}\sum_{i=1}^m L^{(i)}_k,\theta_{D_k},\alpha,\beta_1,\beta_2)$
\EndFor

\EndFor

\State{\texttt{\#\# Generator Optimization}}

\For{\texttt{$i \in \{1,\dots, m\}$}}
\State Sample a complete case $\x$ from $\X_{P_1,:}$ and sample $(K-1)$ noise vectors $\z_k \sim \mathcal{N}(0,\mathbf{I})$ for $k\in[K]\backslash\{1\}$

\For{\texttt{$k \in \{2,\dots,K\}$}}
\State $\widehat{\x}_k\leftarrow G_k(\x,\z,\m_k)$

\State $L^{(i)}_k\leftarrow -D_k(\widehat{\x}_k)-\lambda_2\|\x-\widehat{G}_k(\x,\z,\m_k)\|_1$
\EndFor

\EndFor

\For{\texttt{$k \in \{2,\dots,K\}$}}
\State $\theta_{G_k} \leftarrow \text{Adam}(\nabla_{\theta_{G_k}}\frac{1}{m}\sum_{i=1}^m L^{(i)}_k,\theta_{G_k},\alpha,\beta_1,\beta_2)$
\EndFor

\EndWhile

\State{\texttt{\#\# Imputation}}
\For{\texttt{$k \in \{2,\dots,K\}$}}
\For{\texttt{$i \in P_k$}}
\State Draw a noise vector $\z \sim \mathcal{N}(0,\mathbf{I})$
\State $\X_{i,:} \leftarrow G_k(\X_{i,:},\z,\m_k)$

\EndFor
\EndFor

\State Output imputed data matrix $\X$.

\end{algorithmic}
\end{algorithm}

\subsection{$\text{MI-GAN}_2$ Algorithm}

We notice that \Cref{MI-GAN1} only exploits the information contained in complete cases to train the model. When the number of complete cases is relatively small, the training of $\text{MI-GAN}_1$ is challenging and we may not achieve the optimal generators. Hence the imputation can be negatively affected. To overcome this shortage, we propose an iterative training approach whose model is trained on the whole dataset including incomplete cases. The whole process is presented in \Cref{MI-GAN2}. The iterative training approach requires an initial imputation which can be done by \Cref{MI-GAN1} (if complete cases exist) or other imputation methods. Similar to $\text{MI-GAN}_1$, we create one GAN model for each pattern in $\text{MI-GAN}_2$. Then we train a single GAN model and update imputed values for one pattern at each iteration. The newly imputed values are used for the training of the next GAN model in the next iteration. In more details, when updating the imputed values in the $k$-pattern, we train $G_k$ and $D_k$ on the $\X_{P_{-k},:}$. Then we update $\X_{P_k,:}$ with the newly imputed values from the well-trained generator $G_k$ and this updated $\X_{P_k,:}$ is used to update the $(k+1)$-th pattern.

\begin{algorithm}[]
\caption{$\text{MI-GAN}_2$: iterative imputation}\label{MI-GAN2}
  \textbf{Input:} Initial imputation $\X$, imputation times $M$, burn-in period $N$, thinning step $T$, gradient penalty coefficient $\lambda_1$, reconstruction error penalty $\lambda_2$, initial parameters for the $K-1$ generators $\theta_{G_2},\cdots,\theta_{G_K}$, initial parameters for the $K-1$ discriminators (critics) $\theta_{D_2},\cdots,\theta_{D_K}$, batch size $m$, number of iterations of the critic per generator iteration $n_{\text{critic}}$, Adam hyperparameters $\alpha,\beta_1,\beta_2$ \\
 \textbf{Output:} $M$ Imputed matrix
\begin{algorithmic}[1]

\For{\texttt{$s \in \{1,\dots,N+MT\}$}}

\For{\texttt{$k \in \{2,\dots,K\}$}}
\While {training loss has not converged}

\State{\texttt{\#\# Discriminator Optimization}}
\For{\texttt{$t \in \{1,\dots,n_{\text{critic}}\}$}}

\For{\texttt{$i \in \{1,\dots, m\}$}}
\State Sample two cases $\x,\x'$ from $\X_{P_{-k},:}$, draw a noise vector $\z \sim \mathcal{N}(0,\mathbf{I})$, and a random number $\epsilon\sim U[0,1]$

\State $\widehat{\x}_k\leftarrow G_k(\x,\z,\m_k)$
\State $\widetilde{\x}_k\leftarrow \epsilon \x' +(1-\epsilon) \widehat{\x}_k$

\State $L^{(i)}_k\leftarrow D_k(\widehat{\x}_k)-D_k(\x')+\lambda_1(\|\nabla_{\widetilde{\x}} D(\widetilde{\x})\|_2-1)^2$
\EndFor

\State $\theta_{D_k} \leftarrow \text{Adam}(\nabla_{\theta_{D_k}}\frac{1}{m}\sum_{i=1}^m L^{(i)}_k,\theta_{D_k},\alpha,$\\$\beta_1,\beta_2)$
\EndFor

\State{\texttt{\#\# Generator Optimization}}
\For{\texttt{$i \in \{1,\dots, m\}$}}
\State Sample a case $\x$ from $\X_{P_{-k},:}$ and draw a noise vector $\z \sim \mathcal{N}(0,\mathbf{I})$

\State $\widehat{\x}_k\leftarrow G_k(\x,\z,\m_k)$

\State $L^{(i)}_k\leftarrow -D_k(\widehat{\x}_k)-\lambda_2\|\x-\widehat{G}_k(\x,\z,\m_k)\|_1$
\EndFor

\State $\theta_{G_k} \leftarrow \text{Adam}(\nabla_{\theta_{G_k}}\frac{1}{m}\sum_{i=1}^m L^{(i)}_k,\theta_{G_k},\alpha,\beta_1,\beta_2)$

\EndWhile

\State{\texttt{\#\# Imputation}}
\For{\texttt{$i \in P_k$}}
\State Draw a noise vector $\z \sim \mathcal{N}(0,\mathbf{I})$
\State $\X_{i,:} \leftarrow G_k(\X_{i,:},\z,\m_k)$

\EndFor

\EndFor

\If{$s>N$ and $T\big|(s-N)$}
 \State output $\X$ 
 \EndIf

\EndFor

\end{algorithmic}
\end{algorithm}
















\begin{table*}[!htb]
\vspace{-0.08in}
	\centering
	\begin{tabular}{|c|c|c|c|c|c|c|c|c}
	\hline 
	\text{Models} &\text{Style} &\text{Time(s)} & \text{Imp MSE} & \text{Rel Bias}($\hat\beta_1$) & \text{CR($\hat\beta_1$}) & SE($\hat\beta_1$) & SD($\hat\beta_1$) \\ 
	\hline 
	
	SoftImpute &SI &\textbf{7.3} &\textbf{0.020} &-0.091 &0.78  &0.119 &0.162 \\

    GAIN &SI &39.0 &0.868 & 0.625 & 0.18 &0.146 &0.542 \\
    
    Linear RR & SI &3134.7 &0.066 & 0.148 & 1.00 &0.178 &0.101 \\
    
	MICE & MI & 37.6 & \textbf{0.023} & \textbf{-0.006} &\textbf{0.93}  &0.116 &0.121 \\
    	
    Sinkhorn &MI &31.2 &0.075 &\textbf{0.021} &\textbf{0.96} &0.186 &0.163 \\
    
    $\text{MI-GAN}_1$ &MI &\textbf{3.7} &0.066 &\textbf{0.027} &0.91 &0.157 &0.157 \\
    
    $\text{MI-GAN}_2$ &MI & \textbf{8.0} &\textbf{0.056} &\textbf{0.062} &\textbf{0.94} &0.151 &0.145 \\
     
    
	\hline 
    Complete data &-&-&-&\textbf{-0.003} & \textbf{0.93} &0.109 &0.114\\
	Complete case &-&-&-&0.248 & 0.88 &0.340 & 0.330 \\
	ColMean Imp &SI&-& 0.141 & 0.349 &0.72 &0.221 &  0.172\\
	\hline 
	\end{tabular}
\caption{Blockwise missing data with $n=200$ and $p=251$ under MAR. Approximately $\textbf{40\%}$ features and $\textbf{90\%}$ cases contain missing values. Detailed simulation setup information is in \Cref{section: syn}. Good performance is highlighted in bold.}
\label{p250 Gaussian MAR}
\end{table*}

\begin{table*}[!htb]
\vspace{-0.08in}
	\centering
	\begin{tabular}{|c|c|c|c|c|c|c|c|c}
	\hline 
	\text{Models} &\text{Style} &\text{Time(s)} & \text{Imp MSE} & \text{Rel Bias}($\hat\beta_1$) & \text{CR($\hat\beta_1$}) & SE($\hat\beta_1$) & SD($\hat\beta_1$) \\ 
	\hline 
	
	SoftImpute &SI &12.9 &\textbf{0.028} & -0.246 &0.57 &0.137 &0.179 \\

    GAIN &SI &48.5 &0.790 & 0.697 & 0.25 &0.109 &0.727 \\
    
    
	MICE & MI & 32.7 & \textbf{0.026} & \textbf{-0.032} &0.90  &0.118 &0.141 \\
    	
    Sinkhorn &MI &99.8 &0.100 &-0.193 & 0.88 &0.278 &0.326 \\

    $\text{MI-GAN}_1$  &MI &\textbf{3.7} &0.076 &\textbf{-0.004} &0.89 &0.188 &0.227 \\
    
   $\text{MI-GAN}_2$  &MI & \textbf{8.4} &\textbf{0.048} &\textbf{0.025} &\textbf{0.96} &0.147 &0.146 \\

	\hline 
    Complete data &-&-&-&\textbf{-0.007} & \textbf{0.94}  &0.111 &0.114\\
	Complete case &-&-&-&0.244 & 0.88 &0.376 & 0.394 \\
	ColMean Imp &SI&-& 0.135 & 0.050 &0.95 &0.357 & 0.320\\
	\hline 
	\end{tabular}
\caption{Blockwise missing data with $n=200$ and $p=501$ under MAR. Approximately $\textbf{40\%}$ features and $\textbf{91\%}$ cases contain missing values. Detailed simulation setup information is in \Cref{section: syn}. Good performance is highlighted in bold.}
\label{p500 Gaussian MAR}
\end{table*}

\begin{table*}[!htb]
\vspace{-0.08in}
	\centering
	\begin{tabular}{|c|c|c|c|c|c|c|c|c}
	\hline 
	\text{Models} &\text{Style} &\text{Time(s)} & \text{Imp MSE} & \text{Rel Bias}($\hat\beta_1$) & \text{CR($\hat\beta_1$}) & SE($\hat\beta_1$) & SD($\hat\beta_1$) \\ 
	\hline 
	
	SoftImpute &SI &28.7 &\textbf{0.052} &-0.264 &0.62 &0.166 &0.201 \\
    
    GAIN &SI &114.7 &0.762 & 0.906 & 0.28 &0.094 &0.720 \\
    	
    Sinkhorn &MI &147.9 &0.111 &-0.070 & 0.98 &0.342 &0.290 \\

    $\text{MI-GAN}_1$  &MI &\textbf{4.7} &0.099 &\textbf{-0.035} &\textbf{0.95} &0.239 &0.233 \\
    
   $\text{MI-GAN}_2$  &MI & \textbf{12.6} &\textbf{0.060} &\textbf{0.032} &\textbf{0.95} &0.162 &0.162 \\

   \hline 
    Complete data &-&-&-&\textbf{-0.010} & \textbf{0.96}  &0.111 &0.114\\
	Complete case &-&-&-&0.311 & 0.89 &0.442 & 0.462 \\
	ColMean Imp &SI&-& 0.114 & -0.005 &1.00 &0.351 & 0.282\\
	\hline 
	\end{tabular}
\caption{Blockwise missing data with $n=200$ and $p=1501$ under MAR. Approximately $\textbf{40\%}$ features and $\textbf{92\%}$ cases contain missing values. Detailed simulation setup information is in \Cref{section: syn}. Good performance is highlighted in bold.}
\label{p1500 Gaussian MAR}
\end{table*}

\section{Experiments}

In this section, we validate the performance of MI-GANs through extensive synthetic and real data analysis. In all experiments, we not only evaluate the imputation performance but also quantitatively measure the inference ability of MI-GANs with other state-of-the-art imputation methods.
Given incomplete dataset, we first conduct SI or MI. Then we fit a linear regression on each imputed dataset and compare the regression coefficient estimate $\hat{\bm\beta}$ for each method (Rubin's rule \cite{little2019statistical} is used to obtain the final estimate for MI methods). This step is used to evaluate statistical inference performance. Missing data in all experiments are generated from MAR mechanism. 

We compare $\text{MI-GAN}_1$, $\text{MI-GAN}_2$ with 3 benchmarks: \textbf{Complete data} analysis, \textbf{Complete case} analysis, Column mean imputation (\textbf{ColMean Imp}) and 5 other state-of-the-art imputation
methods: \textbf{MICE} \cite{van2007multiple}, \textbf{GAIN} \cite{yoon2018gain}, \textbf{SoftImpute} \cite{mazumder2010spectral}, \textbf{Sinkhorn} \cite{muzellec2020missing}, and \textbf{Linear RR} \cite{muzellec2020missing}. Specifically, complete data analysis assume there are no missing values and directly fit a linear regression on the whole dataset. Hence, complete data analysis represent the best result of an imputation method can possibly achieve. Complete case analysis does not conduct imputation and fit a linear regression only using the complete cases. Column mean imputation is feature-wise mean imputation. Here, the complete case analysis and column mean imputation, two naive methods, are used to benchmark potential bias and loss of information under MAR mechanism.   

All the experiments run on Google Colab Pro with P100 GPU. For GAIN, Sinkhorn, and Linear RR, we use the open-access implementations provided by their authors, with the default or the recommended hyperparameters in their papers. For SoftImpute, the \texttt{lambda} hyperparameter is selected at each run through cross-validation and grid-point search, and we choose \texttt{maxit=500} and \texttt{thresh=1e-05}. For MICE, we use the  \texttt{iterativeImputer} method in the \texttt{scikit-learn} library with default hyperparameters \cite{pedregosa2011scikit}. For $\text{MI-GAN}_1$, we use default values of $\lambda_1=10$, $\lambda_2=0.1$, $m=256$, $n_{\text{critic}}=5$, $\alpha=0.001$, $\beta_1=0.5$, and $\beta_2=0.9$. For $\text{MI-GAN}_2$, we use default values of $N=3$, $T=1$, $\lambda_1=10$, $\lambda_2=0.1$, $m=256$, $n_{\text{critic}}=5$, $\alpha=0.001$, $\beta_1=0.5$, and $\beta_2=0.9$. Both of $\text{MI-GAN}_1$ and $\text{MI-GAN}_2$ use shallow multilayer perceptrons (MLP) for generators and discriminators. Specifically, $G_k$ use a four-layer ($p\times p\times p\times p$) MLP with \texttt{tanh} activation function and $D_k$ use a four-layer ($p\times p\times p\times 1$) MLP with \texttt{ReLU} activation function. MI methods impute missing values for $10$ times except GAIN and Linear RR. We notice that the GAIN implementation from its original authors conducts only SI and that Linear RR is computationally very expensive.

\subsection{Synthetic data experiments}

In the synthetic data analysis, we generate multiple high-dimensional blockwise missing datasets under MAR mechanism and conduct imputations. Experiment details are included in \Cref{section: syn}. For each imputation method, we report six performance metrics: imputation mean squared error (denoted by \textbf{Imp MSE}), the computing time in seconds per imputation (denoted by \textbf{Time(s)}), relative bias of $\hat\beta_1$ (denoted by \textbf{Rel Bias($\hat\beta_1$)}), standard error of $\hat\beta_1$ (denoted by \textbf{SE($\hat\beta_1$)}), coverage rate of the $95\%$ confidence interval for $\hat\beta_1$ (denoted by \textbf{CR($\hat\beta_1$)}), and standard deviation of $\hat\beta_1$ across 100 MC datasets (denoted by \textbf{SD($\hat\beta_1$)}). Here, $\hat\beta_1$ is one regression coefficient estimate obtained by fitting a linear regression on the imputed datasets. The first two metrics, Imp MSE and Time(s), are used to measure the imputation accuracy and computational cost. Another three metrics, Rel Bias($\hat\beta_1$), SE($\hat\beta_1$) and CR($\hat\beta_1$), are used to assess statistical inference performance. Of note, CR($\hat\beta_1$) that is well below the nominal level of $95\%$ would lead to inflated false positives, an important factor contributing to lack of reproducibility in research. Plus, a well-behaved SE($\hat\beta_1$) should be close to SD($\hat\beta_1$) and a lower SE/SD denotes a less loss of information.

We summarize in \Cref{p250 Gaussian MAR} the results over 100 Monte Carlo (MC) datasets on a four-pattern missing data with $n=200$ and $p=251$. MICE, Sinkhorn, and MI-GANs show small relative bias of $\hat\beta_1$. In addition, MICE, Sinkhorn, and $\text{MI-GAN}_2$ yield nearly nominal level of coverage rate for $\hat\beta_1$. Notably, only four MI methods, MICE, Sinkhorn and MI-GANs, show well-behaved standard errors. Among them, MICE presents best performance in terms of information recovery due to smallest SE. Although SoftImpute presents smallest imputation error, it yields poor statistical inference evidenced by large relative bias and well below coverage rate. Similarly, GAIN and Linear RR lead to poor statistical inference. In terms of computational cost, MI-GANs are the most efficient and Linear RR is the most computationally expensive, preventing it to be applicable to higher dimensional settings.

\Cref{p500 Gaussian MAR} summarizes imputation results on a four-pattern missing data with $n=200$ and $p=501$. Since Linear RR costs too much run-time, it is not presented in this table. As we increase the feature size to $501$, Sinkhorn's performance degenerates significantly and MICE's performance also deteriorates in terms of CR($\hat\beta_1$). In this setting, MICE and MI-GANs show small relative bias of $\hat\beta_1$, and only $\text{MI-GAN}_2$ yields nearly nominal level of CR($\hat\beta_1$). We observe that $\text{MI-GAN}_2$ yields much smaller imputation MSE than $\text{MI-GAN}_1$, benefiting from its iterative training. \Cref{p1500 Gaussian MAR} summarizes imputation results on a four-pattern missing data as we further increase the feature size to $1501$. MICE is not presented due to running out of RAM. At the same time, MI-GANs (especially $\text{MI-GAN}_2$) yield satisfactory results in an efficient manner.

\subsection{ADNI data experiments}

In the real data analysis, we further evaluate the performance of MI-GANs on a large-scale Alzheimer’s Disease Neuroimaging Initiative (ADNI) dataset, which includes both imaging and gene expression data. The original dataset contains 649 cases; each case contains more than 19000 features and a continuous response variable --- the VBM right hippocampal volume. After standardizing each feature, 1000 features that we are interested in are selected as the experiment dataset; among them, three features, which have the maximal correlation with the response variable, are selected as predictors for the subsequent linear regression. 

Here we summarize the results over 100 repeats. Experiments details is included in \Cref{section: adni}. Notice that we have no access to the true regression coefficient $\bm\beta$ in real data analysis. Hence we instead report four metrics: \textbf{Imp MSE}, \textbf{Time(s)}, \textbf{$\hat\beta_1$} and \textbf{SE($\hat\beta_1$)} for each method.

\Cref{adni MAR} presents the results for this dataset. MICE and Linear RR are not presented due to running out of RAM. Although SoftImpute yields the smallest imputation MSE, its $\hat\beta_1$ estimate is far away from the golden standard (which is the $\hat\beta_1$ estimate from complete data analysis). Besides GAIN, $\text{MI-GAN}_1$ and $\text{MI-GAN}_2$ yields the $\hat\beta_1$ estimate closest to the golden standard, which shows MI-GANs can lead to good statistical inference. However, GAIN yields much higher and unacceptable imputation MSE than the naive approach, ColMean Imp, indicating that GAIN is not regarded as a good imputation method in this high-dimensional setting. In addition, MI-GANs is the most computationally efficient, compared to all other state-of-the-art methods. Taking all metrics into consideration, MI-GANs are overall the most powerful imputation method on this setting.

\begin{table}[H]

	\centering
	\begin{tabular}{|c|c|c|c|c|c|c|c}
	\hline 
	\text{Models} &\text{Style} &\text{Time(s)} & \text{Imp MSE} & $\hat\beta_1$ Value  & SE($\hat\beta_1$)  \\ 
	\hline

	SoftImpute  &SI &146.8 &\textbf{0.057} &0.027 &0.012 \\

	GAIN &SI &84.4 &1.264 & \textbf{0.018} &0.009 \\

    Sinkhorn &MI &219.2 &0.076 & 0.026 &0.012 \\
   
	$\text{MI-GAN}_1$ &MI &\textbf{6.3} &\textbf{0.063} & 0.025 &0.011 \\
	
	$\text{MI-GAN}_2$ &MI &\textbf{17.3} &\textbf{0.075} &\textbf{0.022} &0.012  \\

	\hline 
   Complete data &-&-&-&\textbf{0.016} &0.008  \\
	Complete case &-&-&-&0.025 & 0.017  \\
    ColMean Imp &SI&-&0.177 & \textbf{0.021} & 0.013 \\
	\hline 
	\end{tabular}
\caption{Real data experiment with $n=649$ and $p=1001$ under MAR. Approximately $\textbf{40\%}$ features and $\textbf{75\%}$ cases contain missing values. Linear RR and MICE are not included as they run out of RAM. Detailed experiment information is in \Cref{section: adni}. Good performance is highlighted in bold.}
\label{adni MAR}
\end{table}

\section{Discussion}
In this work, we propose a novel GAN-based multiple imputation method, MI-GANs, which can handle high-dimensional blockwise missing data with theoretical support under MAR/MCAR mechanism. Our experiments demonstrate that MI-GANs compete with current state-of-the-art imputation methods and outperform them in the sense of statistical inference and computational speed. One limitation of MI-GANs is that GAN's training is challenging and the generators may not converge when the training sample size is too small. One potential research interest is applying graph neural networks in the generators and the discriminators to reduce the number of parameters when the knowledge graph is available. This may help the generators converge to the optimal point and learn the true conditional distribution.

\bibliographystyle{plain}
\bibliography{ref}

\appendix
\subsection{Synthetic data experiments}
\label{section: syn}

Each MC dataset contains $n=200$ samples and each sample contains $p$ features including $(p-1)$ predictors or auxiliary variables $\X=(\x_1,\dots,\x_{p-1})$ and a response variable $y$. Here we consider settings where $p=251$, $p=501$ and $p=1501$. $\X=(\x_1,\dots,\x_{p-1})$ is generated by reordering variables $\mathbf{A}=(\mathbf{a}_1,\dots,\mathbf{a}_{p-1})$. $\mathbf{A}$ is a first order autoregressive model with autocorrelation $\rho=0.9$, white noise $\epsilon \sim \mathcal{N}(0,0.1^2)$ and $\mathbf{a}_1 \sim \mathcal{N}(0,1)$. Given a variable vector $\mathbf{A}$, $\X$ is obtained by firstly moving $a_{5k+4}$ ($k\in\mathbb{N}$) to the right, secondly moving $a_{5k+5}$ ($k\in\mathbb{N}$) to the right (for example, if $p=11$, $(\mathbf{a}_1,\mathbf{a}_{2},\mathbf{a}_{3},\mathbf{a}_{4},\mathbf{a}_{5},\mathbf{a}_{6},\mathbf{a}_{7},\mathbf{a}_{7},\mathbf{a}_{9},\mathbf{a}_{10})$ becomes $\X = (\mathbf{a}_1,\mathbf{a}_{2},\mathbf{a}_{3},\mathbf{a}_{6},\mathbf{a}_{7},\mathbf{a}_{8},\mathbf{a}_{4},\mathbf{a}_{9},\mathbf{a}_{5},\mathbf{a}_{10})$). Given $\X$, the response $\y$ is generated by      
\begin{align}\label{eq: y}
    \y=\beta_1\cdot \x_{q[1]}+ \beta_2\cdot \x_{q[2]}  + \beta_3\cdot \x_{q[3]}  + \mathcal{N}(0,\sigma_1^2)
\end{align} 
where $\beta_i =1$ for $i\in \{1,2,3\}$ and $q$ is the predictor set. For $p=251,501,1501$, the corresponding predictor sets are $\{210,220,230\}$, $\{380,400,420\}$, and $\{1100,1200,1300\}$. Missing values are separately generated in $\{\x_{\frac{3}{5}(p-1)+1},\dots,\x_{\frac{4}{5}(p-1)}\}$ and $\{\x_{\frac{4}{5}(p-1)+1},\dots,\x_{p-1}\}$ from MAR. Specifically, suppose their missing indicators are $\mathbf{R}_1$ and $\mathbf{R}_2$, then     
\begin{align}
\label{eq: MAR indicator1}
 \text{logit}(\mathbb{P}(\mathbf{R}_1=1|\X,\y))
  &= 1 -2 \cdot\frac{5}{3(p-1)}\sum_{j=1}^{3(p-1)/5}\x_{j}  +3\cdot\y \\
  \label{eq: MAR indicator2}
    \text{logit}(\mathbb{P}(\mathbf{R}_2=1|\X,\y))
  &=2\cdot\frac{5}{3(p-1)}\sum_{j=1}^{3(p-1)/5}\x_{j}  -2\cdot\y
\end{align} 
Here $\mathbf{R}_i=1$ indicates the corresponding group of variables is missing.

\subsection{ADNI data experiments}
\label{section: adni}

\subsubsection{Data Availability}
The de-identified ADNI dataset is publicly available at \href{http://adni.loni.usc.edu/}{http://adni.loni.usc.edu/}. 

\subsubsection{Experiment details}
The original large-scale dataset contains 649 samples and each sample contains 19823 features including a response variable ($\y$), the VBM right hippocampal volume. We prepocess features except response $\y$ by removing their means. Then we rearrange these features in the decreasing order of correlation with $\y$ and only select the first 1000 features, namely $\X=(\x_1,\cdots,\x_{1000})$, to analyze. For each repeat of experiment, we randomly generate missing values in two groups: $\{\x_1,\dots,\x_{200}\}$ and $\{\x_{201},\dots,\x_{400}\}$. Their missing indicators $\mathbf{R}_1,\mathbf{R}_2$ are generated from MAR:   
\begin{align*}
    \text{logit}(\mathbb{P}(\mathbf{R}_1=1)) &=-1-\frac{3}{100}\sum_{j=401}^{500}\x_{j}+3\y \\
   \text{logit}(\mathbb{P}(\mathbf{R}_2=1)) &=-1-\frac{3}{100}\sum_{j=601}^{700}\x_{j}+2\y 
\end{align*}
Here $\mathbf{R}_i=1$ indicates the corresponding group of variables is missing. After imputing the missing values, we fit a linear regression $\mathbb{E}[\y|\x_1,\x_2,\x_3] = \beta_0 + \beta_1 \x_1+\beta_2\x_2+\beta_3\x_3$ and analyze the coefficient $\beta_1$.

\end{document}